\documentclass[letterpaper, 10 pt, conference]{ieeeconf}  
\IEEEoverridecommandlockouts                              

\usepackage{verbatim}
\usepackage{amssymb}
\usepackage{mathtools}
\usepackage{commath}
\usepackage{nth}
\usepackage{cite}
\usepackage{cleveref}
\usepackage{caption}
\usepackage{subcaption}

\usepackage[vlined,ruled]{algorithm2e}

\makeatletter
\let\NAT@parse\undefined
\makeatother

\usepackage[square, sort&compress, numbers]{natbib}

\usepackage{amsmath,graphicx,epsfig,color,amsfonts, }

\graphicspath{{./fig/}}

\def\prob{\mathbb{P}}
\def\expt{\mathbb{E}}
\def\real{\mathbb{R}}

\def\natural{\mathbb{N}}
\newcommand{\until}[1]{\{1,\dots, #1\}}
\newcommand{\subscr}[2]{#1_{\textup{#2}}}

\newcommand{\setdef}[2]{\{#1 \; | \; #2\}}
\newcommand{\seqdef}[2]{\{#1\}_{#2}}

\newcommand{\map}[3]{#1: #2 \rightarrow #3}

\newcommand{\indicator}[1]{\mathbf{1}\left\{#1\right\}}
\newcommand{\ceil}[1]{\left\lceil #1 \right\rceil}

\newcommand{\diag}[1]{\text{diag} \left\{#1\right\}}

\DeclareMathOperator*{\argmax}{arg\,max}

\newcommand\oprocendsymbol{\hbox{$\square$}}
\newcommand\oprocend{\relax\ifmmode\else\unskip\hfill\fi\oprocendsymbol}

\newcommand\bit[1]{\textit{\textbf{#1}}}

\renewcommand{\baselinestretch}{1}

\def \bs {\boldsymbol}
\def \mc {\mathcal}

\def \etal {\emph{et al.}}

\newtheorem{theorem}{Theorem}

\newtheorem{lemma}[theorem]{Lemma}
\newtheorem{corollary}[theorem]{Corollary}

\newtheorem{remark}{Remark}

\renewcommand{\theenumi}{(\roman{enumi}}

\usepackage{caption}
\title{Expedited Multi-Target Search with Guaranteed Performance \\
	via Multi-fidelity Gaussian Processes
\thanks{This work was supported by NSF Award IIS-1734272}
}

\author{Lai Wei, Xiaobo Tan, and Vaibhav Srivastava
\thanks{The authors are with the Department of Electrical and Computer Engineering. Michigan State University, East Lansing, MI 48823 USA.
        {\tt\small e-mail: \{weilai1, xbtan, vaibhav\}@msu.edu}}%
}

\begin{document}

\maketitle
\thispagestyle{empty}
\pagestyle{empty}

\begin{abstract}
We consider a scenario in which an autonomous vehicle equipped with a downward facing camera operates in a 3D environment and is tasked with searching for an unknown number of stationary targets on the 2D floor of the environment. The key challenge is to minimize the search time while ensuring a high detection accuracy. We   model   the   sensing   field   using   a   multi-fidelity   Gaussian process that systematically describes the sensing information available at different altitudes from the floor.  Based on the sensing model, we design a novel algorithm called Expedited Multi-Target Search (EMTS) that (i) addresses the coverage-accuracy trade-off: sampling at locations farther from the floor provides wider field of view but less accurate measurements, (ii) computes an occupancy map of the floor within a prescribed accuracy and quickly eliminates unoccupied regions from the search space, and (iii) travels efficiently to collect the required samples for target detection. We  rigorously analyze the algorithm and establish formal guarantees on the target detection accuracy and the expected detection time. We illustrate the algorithm using a simulated multi-target search scenario. 

\end{abstract}

\section{Introduction}

Autonomous multi-target search requires an autonomous agent to quickly and accurately locate multiple targets of interest in an unknown and uncertain environment. Examples include search and rescue missions, mineral exploration, and tracking natural phenomena. A key challenge in a multi-target search task is to balance several trade-offs including explore-vs-exploit: detecting a target with high accuracy versus finding new targets, and  speed-vs-accuracy: quickly versus accurately deciding on the presence of a target. The latter includes fidelity-vs-coverage trade-off: sampling at locations farther from the floor provides a wider field of view but less accurate measurements.


In this paper, we design and analyze a multi-target search algorithm that addresses these trade-offs. In particular, for expedited search of multiple targets, our algorithm leverages multi-fidelity Gaussian processes to capture the fidelity-coverage trade-off,  information-theoretic techniques to efficiently explore the environment, and Bayesian techniques to accurately identify targets and construct an occupancy map.

Search and persistent monitoring problems have been studied  extensively in the literature. Informative path planning is subclass of these problems in which robot trajectories are designed to maximize the information collected along the way-points while ensuring that the distance traveled is within a prescribed budget. Such informative path planning problems are studied in~\cite{NEL-DAP-etal:07,SLS-MS-DR:12,CGC-XL-XD:13,RNS-MS-etal:11,AK-CEG:12}.

Gaussian processes (GPs) are most widely used models for capturing spatiotemporal sensing fields in robotics~\cite{williams2006gaussian, SV-FR-EN-HD:09}. While GP-based approaches have been used extensively, most of them rely on single-fidelity measurements, i.e., the sensing model does not consider  different altitudes at which the measurements can be collected. GP models have also been used extensively to plan informative trajectories for the robots~\cite{AK-CEG:12, AS-AK-CG-WJK:09, krause2008near, JLN-GJP:09, XL-MS:13}. However, most of these works focus on maximizing the reduction in uncertainty of the estimates.

{In the context of target search, the trajectory should be designed to balance the explore-exploit tension---the robot should spend more time at target locations, while learning target locations. There have been some efforts to address such explore-exploit tension within the context of informative path planning~\cite{DES-MS-DS:12, JY-MS-DR:14, VS-FP-FB:11za, VS-PR-NEL:14, hollinger2014sampling, GAH-BE-etal:13, hitz2017adaptive, hitz2014fully, atanasov2014information, meera2019obstacle, sung2019environmental,XL-MS:13}. }

Hollinger~\etal~\cite{GAH-BE-etal:13} study an inspection problem in which the robot needs to classify the underwater surface. They use a combination of GP-implicit surface modeling and sequential hypothesis testing to classify surfaces. Meera~\etal~\cite{meera2019obstacle} study informative path planning for a target search problem. They model target occupancy as a GP and design a heuristic algorithm for target detection that handles trade-offs among information gain, field coverage, sensor performance, and collision avoidance. They illustrate the performance of their algorithm using numerical simulations. Sung~\etal~\cite{sung2019environmental} study the hot-spot identification problem in an environment within the framework of GP multiarmed bandits~\cite{srinivas2012information,PR-VS-NEL:13d}. The multi-target search can be viewed as a hot-spot identification problem in which, instead of global maximum of the field, all locations with value greater than a threshold need to be identified. Such problems have been studied in the multiarmed bandit literature~\cite{chen2014combinatorial,PR-VS-NEL:14h}; however, we are not aware of any such studies in the GP setting. Furthermore, all these works focus on single fidelity measurements, while we focus on multiple fidelities of measurements induced by the altitudes relative to the 2D floor at which the measurements are collected.

In this paper, we design an algorithm for expedited search of unknown number of targets located at the 2D floor of an unknown and uncertain 3D environment. We use autoregressive multi-fidelity GPs~\cite{kennedy2000predicting,kandasamy2016gaussian} to model the likelihood of the presence of a target at a location as computed by a computer vision algorithm using the sample collected at that location at a given altitude. Here, fidelity corresponds to the altitude at which the samples are collected.  A high altitude (low fidelity) sample provide more global but less accurate information  compared with a low altitude (high fidelity) sample. The low fidelity information can be used to quickly find easy-to-detect targets and this enables the robot to focus on high-fidelity information, possibly only in small regions in the environment and consequently, expedite the search. 
{The proposed EMTS algorithm comprises three main modules (i) a sampling  and fidelity planner,  (ii) a classification and region-elimination algorithm to construct occupancy map of the floor and eliminate unoccupied regions from search space, and (iii) a path planner that allows the vehicle to travel efficiently to collect required samples.} The major contributions of this work are: 
\begin{itemize}
	\item We extend the classical informative path planning approach for single-fidelity GPs to multi-fidelity GPs. This novel extension allows for jointly planning for sampling locations and associated fidelity-levels, and thus, addresses the fidelity-coverage trade-off.

    \item We augment the sampling and fidelity planner with a Bayesian classification and region-elimination algorithm that ensures that the targets are identified with a desired accuracy, as well as a Traveling Sales Person (TSP) path planner that enables travel-efficient sampling. 
    
    

    \item We rigorously analyze the interaction of above algorithms and 
establish formal guarantees of the target detection accuracy and expected detection time. To the best of our knowledge, this is the first performance guarantee for GP based planning in terms of expected target detection time, even in the context of single-fidelity GPs.
\end{itemize}

The remainder of the paper is organized as the following. We present a mathematical formulation of our problem in Section~\ref{sec:background}. 
 In Section~\ref{sec:algo}, we present the EMTS algorithm and illustrate it using an underwater victim search scenario in Section~\ref{sec:simulations}. We analyze the performance of EMTS in Section~\ref{sec:analysis} and conclude this work in Section~\ref{sec:conclusions}.


\section{Problem Description}\label{sec:background}


We consider an autonomous vehicle that moves in a 3D environment, e.g., an aerial or an underwater vehicle. We assume that the vehicle either moves with unit speed or hovers at a location. The vehicle is tasked with searching for multiple targets on the 2D floor of the environment. Let  $D \subset \real^2$ be the area of the floor in which the targets may be present. The vehicle is equipped with a fixed camera that points towards the floor. The vehicle travels across the environment and collects images/videos of the floor (samples) from different sampling points. These sampling points may be located at different altitudes relative to the floor of the environment. We assume that no sample is collected during the movement between sampling points to avoid misleading low-quality sensing information. The collected samples are processed with a computer vision algorithm that outputs a score, which corresponds to the likelihood of a target being present, for each frame. An example of such computer vision algorithm is the state of art deep neural network YOLOv$3$~\cite{redmon2018yolov3}. {The score will be used to update the estimate of the sensing output, i.e., the estimated score function $f: D \rightarrow [0,1]$ which will be used to determine the location of the targets. The stochastic model for $f$ is introduced below.} 

\subsection{Multi-fidelity Sensing Model}
GPs are widely used models for spatially distributed sensing outputs. In~\cite{meera2019obstacle}, a GP is used to model the target detection output of a computer vision algorithm. While target presence is a binary event, the computer vision algorithms such as YOLOv3 yield a score which is a function of the saliency and location of the target in the image. GPs are  appropriate models for such score functions. So far in the literature, GPs have been used in the context of single-fidelity measurements. To characterized the inherent fidelity-coverage trade-off in sensing the floor scene by an autonomous vehicle operating in 3D space, we employ a novel multi-fidelity GP model. The two key physical sensing characteristics the model seeks to capture are: (i) there is some information that can only be accessed at lower altitudes, (ii) the sensing outputs are more spatially correlated at higher altitudes, since the fields of view at neighboring locations have higher overlaps in their field of views.


We assume that the vehicle can collect samples of the floor from $M$ possible heights from the floor $z_1 > z_2>\cdots > z_M$. We refer to these heights as the fidelity level of the measurement, with $M$ (resp. $1$) corresponding to the highest (resp. lowest) level of fidelity. Let the score function $\map{g_m}{D}{[0,1]}$ be  defined by the output of the computer vision algorithm for an ideal noise-free image collected at fidelity level $m \in \until{M}$ with the field of view of the camera centered at $\bs x \in D$.
We assume that the score functions for a location $\bs x$ obtained from different altitudes (fidelity levels) are related to each other in an autoregressive manner as follows
\begin{equation}\label{def: mfgp-old}
g^m (\boldsymbol{x}) = a_{m-1} g^{m-1}(\boldsymbol{x}) + b^m (\boldsymbol{x}),
\end{equation}
where $a_{m-1}$ is a scale parameter and $b^m$ is the bias term that captures the information that can be only be accessed at fidelities levels greater than $m$. Let $f^{m}(\bs x) = \left(\prod_{i=m}^{M-1} a_i\right) g^m(\bs x)$ and $h^{m}(\bs x) = \left(\prod_{i=m}^{M-1} a_i\right) b^m(\bs x)$. Then, equation~\eqref{def: mfgp-old} reduces to 
\begin{equation}\label{def: mfgp}
f^m (\boldsymbol{x}) = f^{m-1}(\boldsymbol{x}) + h^m (\boldsymbol{x}), 
\end{equation}
where  $f^0(\bs x) = 0$ and $f(\bs x) := f^M(\bs x)$ is the score function at the highest fidelity level which we treat as ground truth.  We model the influence of systemic errors in sample collection and environmental uncertainty on the output of the computer vision algorithm for an input at fidelity level $m$ through an additive zero mean Gaussian random variable $\epsilon_m$ with variance $s_m^2$, i.e., $\epsilon_m \sim N(0, s_m^2)$. Consequently, the (scaled) score obtained by collecting a sample at location $\bs x$ is a random variable $y = f_m(\boldsymbol{x}) + \epsilon_m$. 

We assume that each $h_m$ is a realization of a Gaussian process with a constant mean $\mu_m$ and a squared exponential kernel function $k^m(\boldsymbol{x},\boldsymbol{x}')$ expressed as
\begin{equation}\label{def: stker}
k^m(\boldsymbol{x},\boldsymbol{x}') = v_m^2 \exp\left(-\frac{\norm{\boldsymbol{x}-\boldsymbol{x}'}^2}{2 l_m^2}\right),
\end{equation}
where $l_m$ is the length scale parameter, and $v_m$ is the variability parameter that satisfies $v_1 > v_2 > \cdots > v_M$.  This kernel function describes the spatial correlation of score function at neighboring locations at each fidelity level. Since the fields of view are more overlapped at lower fidelity levels, it results in $l_1 > l_2 > \cdots>l_M$. 

We assume that for an ideal highest-fidelity sample collected at location $\bs x$, the computer vision algorithm yields a score $f(\boldsymbol{x})$ greater than a threshold $\texttt{th}$, if the target is in the field of view at $(\bs x, z_M)$. 


\subsection{Objective of the Search Algorithm}

Our objective is to design an algorithm for sequentially determining sampling points that lead to expedited detection and localization of targets within a desired accuracy. In particular, the algorithm should classify, each location $\boldsymbol{x} \in D$, as \emph{empty} or \emph{target}, with the probability of misclassification less than $\delta \in (0,1/2)$. 
%
Let $t(\boldsymbol{x},\delta)$ be the total (traveling and sampling) time  until the location $\boldsymbol{x}$ is classified with misclassification rate smaller than $\delta$. Then, the objective of the algorithm is to determine the sequence of sampling points that achieves efficient mean classification time 
\[
 \bar{t}(\boldsymbol{x},\delta) = \; \expt \left[ t(\boldsymbol{x},\delta))\right],
 \]
at each $\bs x \in D$.

\section{Expedited Multi-target Search Algorithm}\label{sec:algo}

 The proposed EMTS algorithm is illustrated in Fig.~\ref{fig.EMTS}. It operates using an epoch-based structure. In each epoch, sampling and fidelity planner computes a set of sampling points and the path planner optimizes a TSP tour going through those points. The vehicle follows the TSP tour to collect measurements at sampling points and the inference algorithm uses these measurements to update the estimate the score function $f$. Then, the Bayesian classification uses these estimates to compute an occupancy map of the floor and the region elimination module removes regions with no target with sufficiently high probability from the search space. In the following, we describe each of these modules in detail.

\begin{figure}[t]
	\begin{center}
		\includegraphics[width=0.45\textwidth]{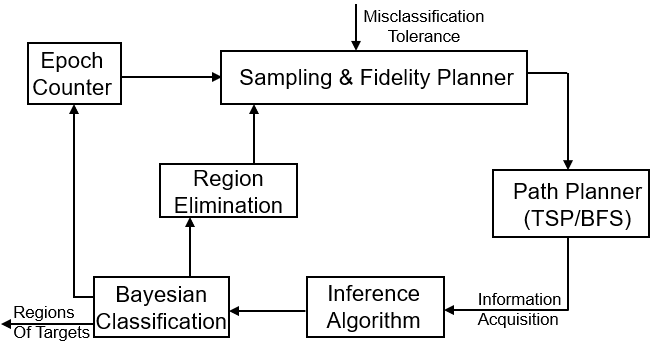}
		\caption{Architecture of EMTS}
		\label{fig.EMTS}		
	\end{center}
\end{figure}

\subsection{Inference Algorithm for Multi-fidelity GPs}

The Bayesian inference method for multi-fidelity GPs  discussed in this section is an extension of the inference procedure in~\cite{kennedy2000predicting} for the case of no sampling noise. Let the set of sampling location-score-fidelity tuples after $n$ observations be $\mathcal{P}_n = \setdef{\left(\boldsymbol{x}_i, y_i, m_i\right)}{i\in\until{n}}$. For each fidelity $m$, define a subset of $\mathcal{P}_n$,
\[P^{m}_n = \setdef{\left(\boldsymbol{x}_i, y_i, m_i\right) \in \mathcal{P}_n}{ m_i = m},\]
and $\lvert P^{m}_n \rvert$ denote the cardinality of $P_n^{m}$. Recall that $k^i(\bs x, \bs x')$ is the kernel function for the GP $h_i$ at $i$-th fidelity level. Let $\boldsymbol{K}^i_0 \big(P_n^{m} , P_n^{m'}\big)$ be a $\lvert P_n^{m} \rvert \times \lvert P_n^{m'}\rvert$ matrix with entries $k^i(\boldsymbol{x}, \boldsymbol{x}'), \ \bs x \in P_n^{m}, \ \bs x' \in P_n^{m'}$ and $\boldsymbol{K}^i_0(P_n^{m}, \boldsymbol{x})$ be a $\lvert P_n^{m}\rvert$ dimensional vector with entries $k^i_0(\boldsymbol{x'}, \boldsymbol{x}), \ \bs x' \in P_n^{m}$. Let 
$\boldsymbol{K}$ be a $M \times M$ block matrix with $\left(m,m'\right)$ block submatrix
\[
\boldsymbol{K}_{m,m'} =   \sum_{i=1}^{\min(m,m')} \boldsymbol{K}_i\big(P_n^{(m)} , P_n^{(m')}\big).
\]Let
$\boldsymbol{k}(\boldsymbol{x})$ be a $\abs{\mathcal{P}_n}$ dimensional vector constructed by concatenating  $M$ sub-vectors $\boldsymbol{k}(\boldsymbol{x}) = \big(\boldsymbol{k}^{1}(\boldsymbol{x}), \ldots, \boldsymbol{k}^{M}(\boldsymbol{x})\big)$,
where
\begin{equation}\label{eq: postupd}
\boldsymbol{k}^{m}(\boldsymbol{x}) = \sum_{i=1}^{m}  \boldsymbol{K}_i(P_n^{m}, \boldsymbol{x}),\quad \forall m \in \until{M}.
\end{equation}
Denoted by $\boldsymbol{\Theta}$ is the $M \times M$ diagonal matrix with variance of sampling noise at diagonal entries
\[
\boldsymbol{\Theta} = \diag{s_m^2 \boldsymbol{I}_{\lvert P_n^{m} \rvert}}_{m=\until{M}}.
\]

Let $\boldsymbol{\nu}_n = \left[ \nu_1,\ldots,\nu_n\right]$ be the a priori mean of the sample $ \boldsymbol{y}_n = \left(y_1,\ldots,y_n\right) $. In particular, if $y_j$ is a sample at fidelity $m$, then $\nu_j = \sum_{i=1}^m \mu_i$. 
The a priori covariance of $\bs y_n$ is $\boldsymbol{K+\Theta}$.  In the training process with training dataset $\mathcal{P}_n$, the hyperparameters $\seqdef{\mu_m, v_m, l_m, s_m}{m =1}^M$ and $\seqdef{a_m}{m=1}^{M-1}$ in the multi-fidelity GP can be learned by maximizing a log marginal likelihood function $-\frac{1}{2} \log \left(\det\left(2\pi\left(\boldsymbol{K} +  \boldsymbol{\Theta}\right)\right)\right) -\frac{1}{2} \left(\boldsymbol{y}- \bs \nu_n \right)^{T} \left(\boldsymbol{K} +  \boldsymbol{\Theta}\right)^{-1} \left(\boldsymbol{y}- \bs \nu_n \right)$.
Such training can be performed using the GP toolbox~\cite{perdikaris2017gaussian}.



Due to the multi-fidelity structure described in~\eqref{def: mfgp-old} and~\eqref{def: mfgp}, the prior mean and covariance of $f$ are
\[\mu_0 (\boldsymbol{x}) = \sum_{m=1}^M  \mu_m, \quad k_0 (\boldsymbol{x},\boldsymbol{x}') = \sum_{m=1}^{M}  k^m (\boldsymbol{x},\boldsymbol{x}').\]
When running EMTS with learned hyperparameters, it can be shown that the posterior mean and covariance functions of $f$ after $n$ measurements are
\begin{align}\label{eq:multi-GP-inference}
\begin{split}
\mu_n (\boldsymbol{x}) &=  \mu_0 (\boldsymbol{x}) + \boldsymbol{k}^T(\boldsymbol{x}) \left(\boldsymbol{K} +  \boldsymbol{\Theta}\right)^{-1}\left(\boldsymbol{y}- \bs \nu_n \right) \\
k_n \left(\boldsymbol{x}, \boldsymbol{x}'\right) &= k_0 \left(\boldsymbol{x}, \boldsymbol{x}'\right) - \boldsymbol{k}^T(\boldsymbol{x}) \left(\boldsymbol{K} + \boldsymbol{\Theta} \right)^{-1} \boldsymbol{k}(\boldsymbol{x}').
\end{split}
\end{align}
Note that the posterior variance $\sigma_n^2 (\boldsymbol{x}) = k_n \left(\boldsymbol{x}, \boldsymbol{x}\right)$ is a measure of uncertainty that will be utilized to classify $\boldsymbol{x}$. It should be noted that the measurements collected at different fidelity levels are appropriately incorporated in the inference~\eqref{eq:multi-GP-inference}.

\subsection{Multi-fidelity Sampling \& Path Planning}
For each epoch $j$, we seek to design an efficient sampling tour through sampling locations $\{(\boldsymbol{x}_{n_j+1}, z_{n_j+1}),\ldots,(\boldsymbol{x}_{n_{j+1}}, z_{n_{j+1}})\}$ to ensure

\[
\frac{\max_{\boldsymbol{x} \in D}  \sigma_{n_{j+1}}(\boldsymbol{x})}{\max_{\boldsymbol{x} \in D}  \sigma_{n_j}(\boldsymbol{x})} \leq \frac{3}{4},
\]
where $n_j$ is the number of samples collected before the beginning of the $j$-th epoch and the uncertainty reduction threshold at $3/4$ is selected based on the analysis discussed in Section~\ref{sec:analysis}.  

Notice that the posterior variance update in~\eqref{eq:multi-GP-inference} depends only on the location of the observations $\boldsymbol{y}_n$, but not on the realized value of $\boldsymbol{y}_n$. Therefore, the sequence of sampling location-fidelity tuples can be computed before physically visiting the locations. Such deterministic evolution of the variance has been leveraged  within the context of single-fidelity GP planning to design efficient sampling tours~\cite{kemna2017multi}. 

\subsubsection{Sampling Point Selection}
The vehicle follows a greedy sampling policy at each fidelity level, i.e., at each sampling round the vehicle selects the most uncertain point as the next sampling point
\begin{equation}\label{def: greedy}
\boldsymbol{x}_{n} = \argmax_{\boldsymbol{x} \in D} \, \sigma_{n-1}(\boldsymbol{x}).
\end{equation}
In the information theoretic view~\cite{AK-CEG:12}, the greedy policy is near optimal in terms of maximizing an appropriate measure of uncertainty reduction (see Section~\ref{sec:analysis}.)


\subsubsection{Fidelity Selection}
For each sampling point $\boldsymbol{x}_{n}$, a fidelity level (or sampling altitude) needs to be assigned. We let the vehicle start at fidelity level $1$ and successively visit all fidelity levels from the lowest to the highest. Since sampling $f^m$ is not able to reduce the uncertainty about $f$ introduced by the subsequent bias terms $h^{m+1},\ldots,h^M$, we define the \emph{inaccessible uncertainty} at fidelity level $m$ as $\xi_m = \sum_{i=m+1}^{M} v^2_i$. Accordingly, we define the \emph{accessible uncertainty} about $f$ at fidelity level $m$ by $r^{m}_n = \max_{\boldsymbol{x}\in D} \sigma^2_n(\boldsymbol{x}) - \xi_m $. The assigned fidelity level to sample point $\boldsymbol{x}_n$ is designed to change from fidelity $m$  to $m+1$ when
\[
{r_n^{m}} \leq \frac{l_{m+1}^2}{l_m^2} v^2_{m+1}.
\]
Notice that before the vehicle begins to sample at fidelity level $m$, $r_n^{m} \geq v_m^2 \geq v^2_{m+1}{l_{m+1}^2}/{l_m^2}$, where the second inequality is due to the assumption that $v_m > v_{m+1}$ and $l_m > l_{m+1}$. This ensures that all fidelity levels are visited from the lowest to the highest successively.

\subsubsection{Path Planning}
Since the order of sampling locations does not influence the eventual posterior mean and variance, the path going through the sampling location can be optimized by computing an approximate TSP tour using packages, such as Concorde~\cite{applegate2006concorde}. Such a tour-based sampling policy allows for energy and time-efficient operation of the vehicle. If all measurements within epoch $j$ are collected at the same fidelity level, the vehicle traverses the TSP tour TSP$(\boldsymbol{x}_{n_j+1},\ldots,\boldsymbol{x}_{n_{j+1}})$ to collect measurements from sampling points and update posterior distribution of $f$. Otherwise, a TSP tour each is designed at every fidelity level.


\subsection{Classification and Region Elimination}
The classification and elimination of regions follows a confidence-bound-based rule, which has been widely used in pure exploration multi-armed bandit algorithms~\cite{audibert2010best} and robotic source seeking~\cite{rolf2018successive}. We extend these ideas to the case of multi-fidelity GP setting considered in this paper. 

Conditioned on $\mc P_n$, the distribution of $f(\bs x)$ is Gaussian with mean function $\mu_n(\bs x)$ and variance $\sigma_n^2(\boldsymbol{x})$. Let $(L_n (\boldsymbol{x},\varepsilon), U_n (\boldsymbol{x},\varepsilon))$ be the Bayesian confidence interval containing $f(\bs x)$  with probability greater than $(1-2\varepsilon)$.  Here, the lower confidence bound $L_n$ and upper confidence bound $U_n$ are defined by 
$
L_n (\boldsymbol{x},\varepsilon) = \mu_n(\boldsymbol{x}) - c(\varepsilon) \sigma_n \left(\boldsymbol{x} \right), \, U_n (\boldsymbol{x},\varepsilon) = \mu_n(\boldsymbol{x}) + c(\varepsilon) \sigma_n\left(\boldsymbol{x} \right),
$
with $c(\varepsilon) = \sqrt{2 \ln\left(1/(2\varepsilon )\right)}$.

Given the desired maximum misclassification rate $\delta$, at the end of epoch $j$, a location $\boldsymbol{x}$ is classified as \emph{target}, if $L_{n_j} \left(\boldsymbol{x},\delta/2^j\right) \geq \texttt{th}$, and is added to $D_t$; while it is classified as \emph{empty}, if $U_{n_j} \left(\boldsymbol{x},\delta/2^j\right) < \texttt{th}$, and is added to the set $D_e$. Note that the confidence parameter $\varepsilon = \delta/2^j$ defining the lower and upper bounds is decreased exponentially with epochs, and we will show that it ensures  a misclassification rate smaller than $\delta$. The locations in the set $D_e$ are removed from sampling space $D$ at the end of each epoch.


Different rules can be used to terminate EMTS, such as giving termination time or setting maximum variance lower bound. In this work, we terminate EMTS when $99\%$ of the regions in $D$ are classified.

\section{An Illustrative Example}\label{sec:simulations}
In this section, we illustrate  EMTS using the Unmanned Underwater Vehicle Simulator~\cite{manhaes2016uuv}, which is a ROS package designed for Gazebo robot simulation environment. We integrate it with YOLOv$3$~\cite{redmon2018yolov3} for image classification and Concorde solver~\cite{applegate2006concorde} to compute TSP tours. We use $2$ fidelity levels situated at $11$m and $5$m from the water floor, respectively. Fig.~\ref{fig.env} shows our simulation setup, where $3$ victims are located at different locations on a $40\text{m} \times 40 \text{m}$ water floor. At each sampling point, the vehicle take $20$ images and YOLOv$3$ returns an average score about the confidence level of the existence of victims in the view.

\begin{figure}[ht!]
	\begin{center}
		\includegraphics[width=0.48\textwidth]{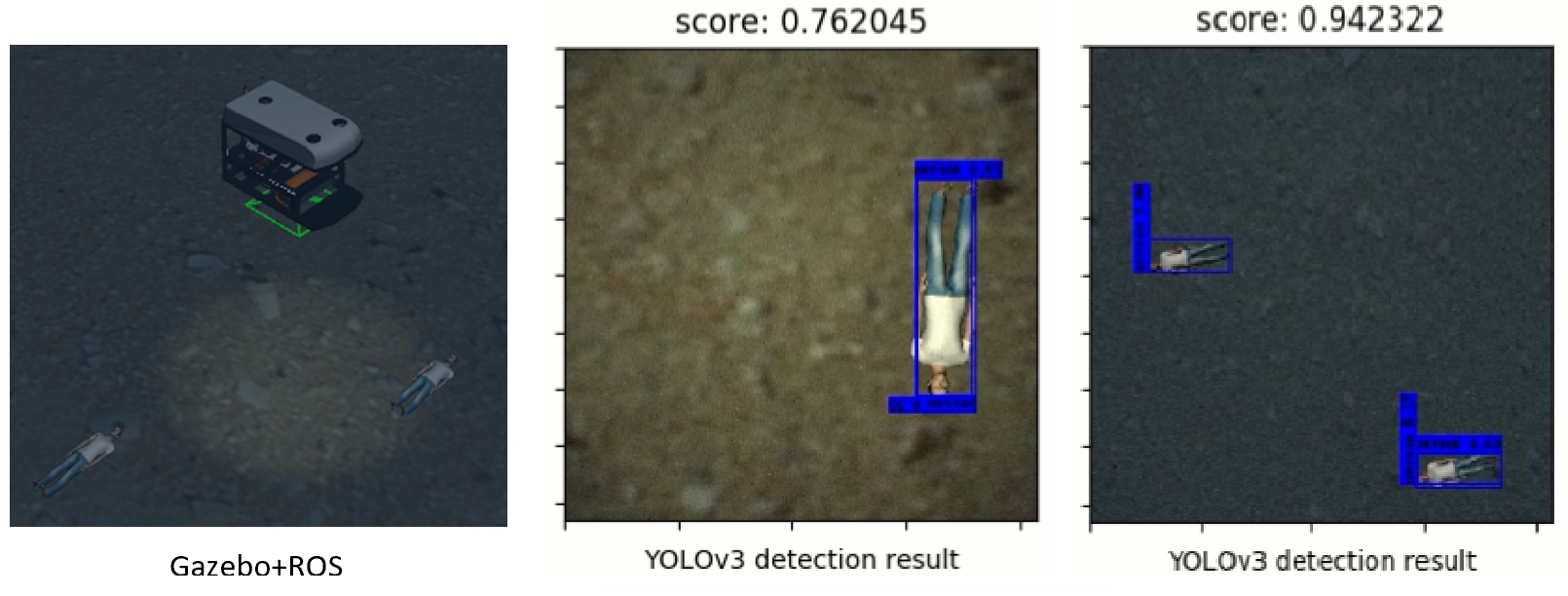}
		\caption{Test Environment: An underwater vehicle is equipped with a downward camera and a flash light to facilitate the searching task in dark underwater environment. Middle figure and right figure are detection result with YOLOv3 at a high fidelity level and a low fidelity level, respectively.}
		\label{fig.env}		
	\end{center}
\end{figure}

Each subplots of Fig.~\ref{key} shows the classification of regions before each epoch, the sampling points selected by the greedy policy and the planned path. Classifications of the environment are represented by $3$ colors: red means target exist, blue means no target, and green means uncertain. The dark green points and lines are the planned sampling locations and paths at the low fidelity level and red points and lines are sampling locations and paths at the high fidelity level. At the beginning of epoch $1$, all regions are classified as uncertain. After first, second and third exploration tours, classified regions increases to $85.5\%$, $98.4\%$ and $99.3\%$, respectively. The detection task is terminated since more than $99\%$ of the regions are classified. Notice that the vehicle switches to the high fidelity level at epoch $2$. The tours at  low and high fidelity levels are plotted using two different colors. The vehicles do not sample in blue regions since they have been classified as empty. In the final result, the regions with target are successfully found. A video of the simulation is available as supplementary material.

\begin{figure}[ht!]
	\centering
	\begin{subfigure}[b]{0.24\textwidth}
		\centering
		\includegraphics[width=\textwidth]{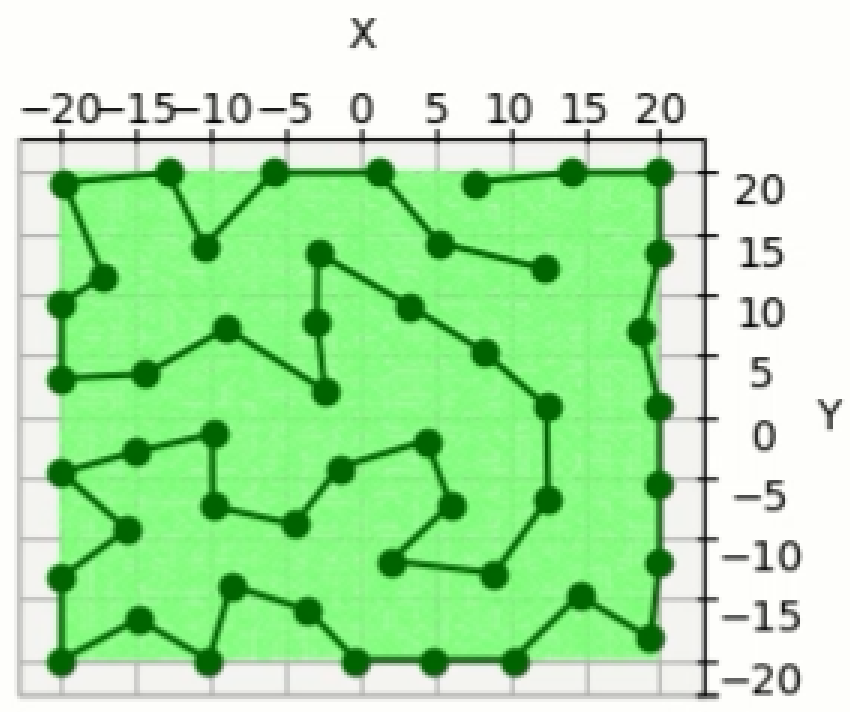}
		\caption{Epoch $1$}\label{1rst}
	\end{subfigure}
	\hfill 
	\begin{subfigure}[b]{0.24\textwidth}
		\centering
		\includegraphics[width=\textwidth]{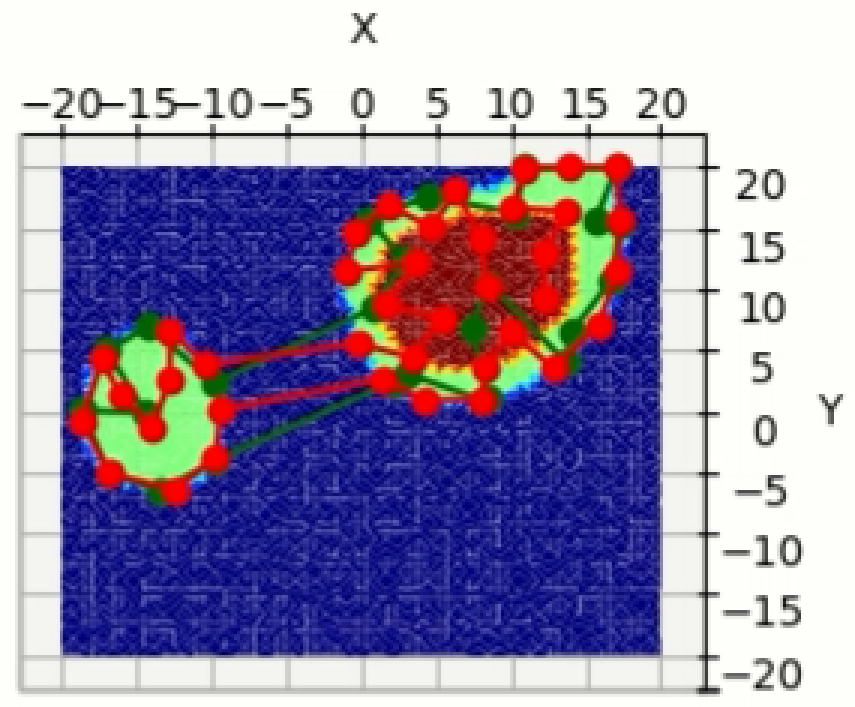}
		\caption{Epoch $2$}\label{2nd}
	\end{subfigure}
	\begin{subfigure}[b]{0.24\textwidth}
		\centering
		\includegraphics[width=\textwidth]{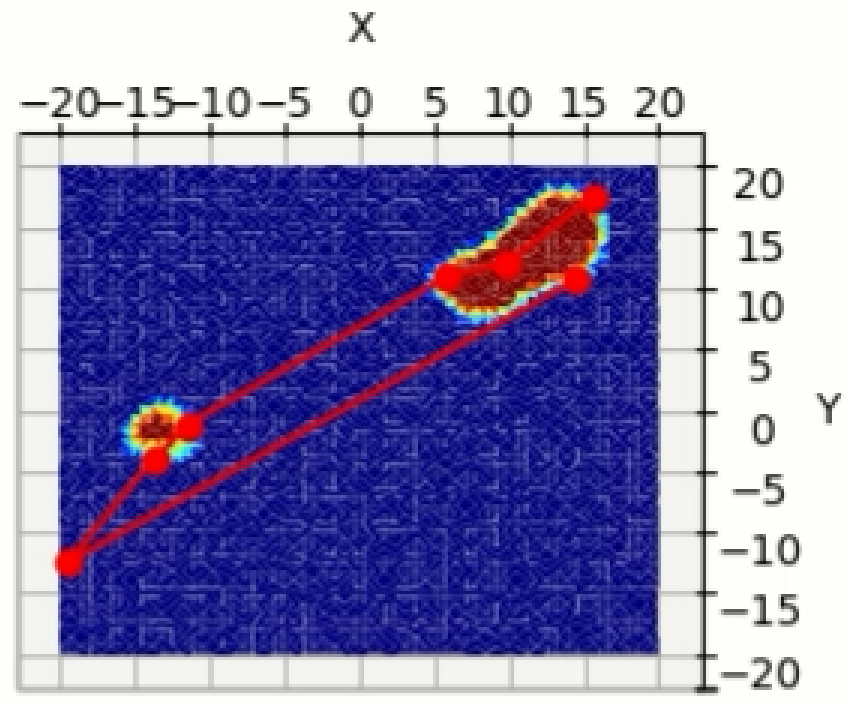}
		\caption{Epoch $3$}\label{3rd}
	\end{subfigure}
	\begin{subfigure}[b]{0.239\textwidth}
		\centering
		\includegraphics[width=\textwidth]{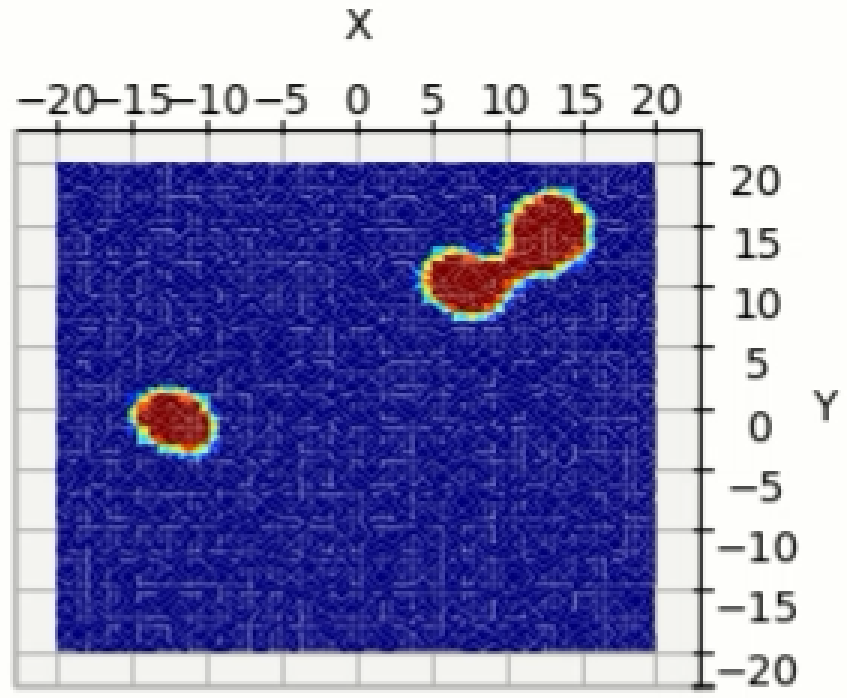}
		\caption{Final result}\label{final}
	\end{subfigure}

	\caption{Performance of EMTS. (i) The green dots and lines are sampled locations and the path traversed by the vehicle at the low fidelity level and the red ones are for the high fidelity level. (ii) Classification results of the environment are represented by $3$ colors: red means target exist, blue means no target, and green means uncertain. (iii) The vehicle switches to high fidelity level at epoch $2$.}\label{key}
\end{figure}

 In Fig.~\ref{post_var}, we show the heat map of posterior variance for the whole region. The regions classified as empty have larger posterior variance since they have been eliminated from sampling space. This shows that EMTS is able to put more focus on areas likely to contain victims. The uncertainty reductions, i.e. the decreases of maximum posterior variance, by doing multi-fidelity greedy sampling and single-fidelity greedy sampling, are compared in Fig.~\ref{var_cov}. It shows that greedy multi-fidelity sampling can reduce uncertainty much faster at the beginning stage, which will enable EMTS to eliminate unoccupied regions quickly, and hence, accelerate target search.

\begin{figure}[ht!]\label{pos}
	\centering
	\begin{subfigure}[b]{0.25\textwidth}
		\centering
		\includegraphics[width=\textwidth]{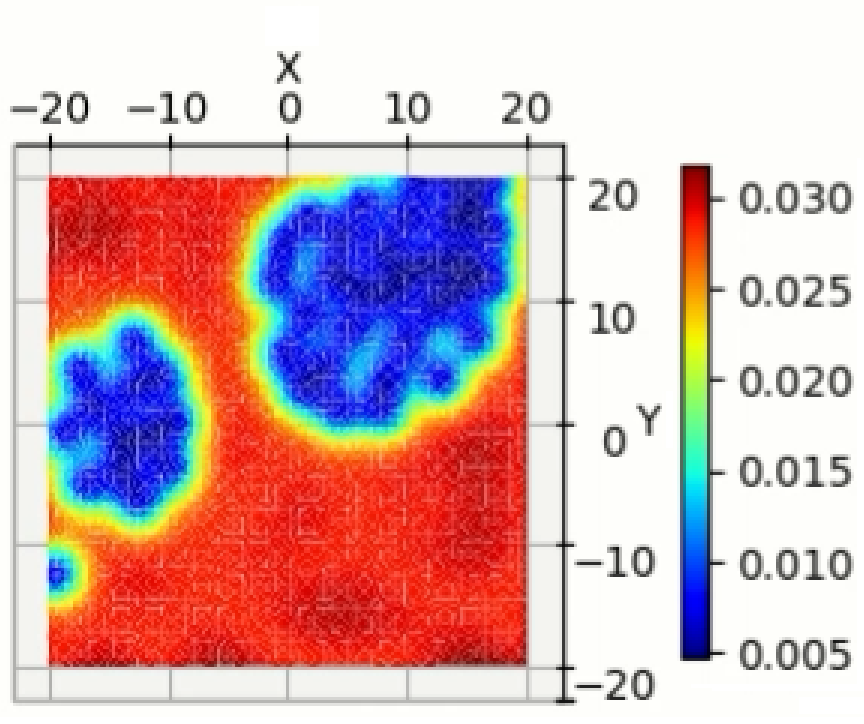}
		\caption{Final posterior variance}\label{post_var}
	\end{subfigure}
	\hfill 
	\begin{subfigure}[b]{0.23\textwidth}
		\centering
		\includegraphics[width=\textwidth]{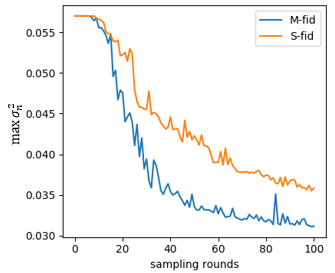}
		\caption{Convergence of $\sigma_n^2$}\label{var_cov}
	\end{subfigure}
	\caption{Uncertainty reduction results. (a) shows spatial nature of uncertainty reduction with EMTS, i.e., the posterior variance is low only at areas that likely contain a target. (b) shows the temporal nature of uncertainty reduction by comparing the decreasing speed of posterior variance with multi-fidelity greedy sampling and single fidelity greedy sampling. } 
\end{figure}

\section{Analysis of EMTS}\label{sec:analysis}
In this section, we analyze the modules of the EMTS algorithm and use these analyses to derive an upper bound on the expected detection time for the overall algorithm. 

\subsection{Analysis of the classification algorithm}

We first characterize the Bayesian confidence interval for $f(\bs x)$, and then use this result to establish that the EMTS algorithm ensures a desired classification accuracy.

\begin{lemma}[\bit{Bayesian confidence interval}]\label{lemma: conineq}
	For $f(\boldsymbol{x}) \, |\, \mathcal{P}_n \sim N \left(\mu_n(\boldsymbol{x}),\sigma_n^2(\boldsymbol{x}) \right)$ and $\varepsilon \in (0,1/2) $,
	\begin{equation*}\label{def: vlprob}
	\prob\left(f(\boldsymbol{x}) \leq L_n (\boldsymbol{x},\varepsilon) \right) = \prob\left(f(\boldsymbol{x}) \geq U_n (\boldsymbol{x},\varepsilon) \right) \leq \varepsilon.
	\end{equation*}
\end{lemma}

\smallskip

\begin{proof}
	To normalize $f(\boldsymbol{x})$, let $r = \big(f(\boldsymbol{x}) - \mu(\boldsymbol{x})\big)/{\sigma(\boldsymbol{x})}$ and $c (\varepsilon)= \sqrt{2 \ln\left({1}/(2\varepsilon) \right)}.$
	Now $r \sim N(0,1)$, and from tail-inequality for standard normal distribution~\cite{MA-IAS:64}
	\begin{align*}
	\prob \left( r \geq c\right) 
	&\leq \frac{1}{2}\exp \left(  -\frac{c^2}{2} \right) = \varepsilon,
	\end{align*}
	which prove the $\prob\left(f(\boldsymbol{x}) \geq U_n (\boldsymbol{x},\varepsilon) \right) \leq \varepsilon$. Similar result holds for lower confidence bound.
\end{proof}
\begin{theorem}[\bit{Misclassification Rate}]\label{theorem: miscl}
	For the classification strategy in the EMTS algorithm, a location $\boldsymbol{x} \in D$ is misclassified with probability at most equal to $\delta$.
\end{theorem}

\begin{proof}
Consider a location $\boldsymbol{x}$ such that $f(\boldsymbol{x}) \leq \texttt{th}$, i.e., the true classification of $\bs x$ is \emph{empty}. Since at the end of epoch $j$, the lower and upper confidence bounds used for classification employ $\varepsilon = \delta/2^j$, we apply a union bound to show the probability of classifying $\boldsymbol{x}$ as a \emph{target} satisfies
\begin{align*}
	\sum_{j=1}^{\infty} \prob\left( L_{n_j} (\boldsymbol{x},\delta/2^j) > \texttt{th}\right)
	&\leq \sum_{j=1}^{\infty} \prob\left( L_{n_j}(\boldsymbol{x},\delta/2^j) > f(\boldsymbol{x})\right).
\end{align*}
Then, it follows from Lemma~\ref{lemma: conineq} that the misclassification probability is no greater than $\sum_{j=1}^{\infty} {\delta}/{2^j} = \delta$. The case of location $\bs x$ being occupied by a target follows similarly.
\end{proof}

\subsection{Analysis of the Sampling and Fidelity Planner}

We now analyze the information gain and uncertainty reduction properties for our sampling and fidelity planner. We first recall some results for the single fidelity planner and then extend them to the case of multi-fidelity planner.


Consider a single-fidelity GP $f$ that is sampled with additive Gaussian noise with variance $s^2$. Let $X_n$ be the set of first $n$ sampling points and let the vector of associated observations be $\boldsymbol{y}_{X_n}$. 
It is shown in~\cite[Lemma 5.3]{srinivas2012information} that the mutual information between $\boldsymbol{y}_{X_n}$ and $f$ is
\begin{equation}
I \left(\boldsymbol{y}_{X_n} ; f \right) = \frac{1}{2}  \sum_{i=1}^{n} \log \left( 1 + s^{-2}\sigma_{i-1}^2\left(\boldsymbol{x}_i\right) \right),\label{eq:mutual-info}
\end{equation}
where $\boldsymbol{{f}}_{X_n}$ is the vector of $f(\bs x)$ calculated at points in $X_n$.  Let the maximal mutual information gain with $n$ samples be
\[
\gamma_n := \max_{Z \in D : \, \abs{Z} = n} I\left(\boldsymbol{y}_Z ; f\right). 
\]
Let $\subscr{I}{greedy}$ be the total mutual information gain using a greedy policy that maximizes the summand in ~\eqref{eq:mutual-info} at each sampling step. It follows, due to submodularity~\cite{nemhauser1978analysis} of $I \left(\boldsymbol{y}_{X_n} ; f \right)$, that 
\[
\left(1-\frac{1}{e}\right) \gamma_n \leq \subscr{I}{greedy} \left(\boldsymbol{y}_{X_n} ; f \right) \leq \gamma_n,
\]
While giving an exact value of $\gamma_n$ is difficult, an upper bound on $\gamma_n$  for squared exponential kernel derived in~\cite{srinivas2012information} is presented in the following Lemma~\ref{lemma: parmi}.

\begin{lemma}[\bit{Information gain for squared exp. kernel}] \label{lemma: parmi}
	Let a GP $f$ be defined on domain $D \subset \real^2$. If $f$ has squared exponential kernel with length scale $l$, then the maximum mutual information satisfies
	\[
	\gamma_n(l) \in O( l^{-2}(\log n)^3).
	\]
\end{lemma}

\begin{proof}
For a GP defined on $D\in [0,1]^2$ with squared exponential kernel function $k(\boldsymbol{x},\boldsymbol{x}') = \exp(-\|{\boldsymbol{x}-\boldsymbol{x}'}\|^2/2)$, $\gamma_n \in O((\log n)^{3})$~\cite{srinivas2012information}.
		It is shown in~\cite{kandasamy2016gaussian} that $\gamma_n$ scales with the area of $D$. Thus, if the diameter of $D$ is $d$, then $\gamma_n \in O\left(d^2 (\log n)^3 \right)$. Note that having length scale $l$ in kernel function is equivalent to scale $D$ by $1/l$. Accordingly, $\gamma_n \in O\left(d^2 l^{-2}(\log n)^3\right)$. For fixed $D$, we omit diameter $d$ from the order notation and write $\gamma_n(l) \in O( l^{-2} (\log n)^3)$.
\end{proof}

Lemma~\ref{lemma: parmi} provides a bound on the mutual information gain at the first fidelity level. For higher fidelity levels, the Gaussian process is composed of summation of independent GPs. We now  establish  that the information gained by sampling the sum of GPs is smaller than the information gained by sampling them independently, and then use this result to establish the bound on information gain for multi-fidelity GPs. 
\begin{lemma}[\bit{Information gain for sum of GPs}]\label{lemma: compmi}
	Let $h_1 \sim GP(\mu_1(\boldsymbol{x}), k_1(\boldsymbol{x},\boldsymbol{x}') ) $ and $h_2 \sim GP(\mu_2(\boldsymbol{x}), k_2(\boldsymbol{x},\boldsymbol{x}') ) $ be independent GPs. Consider a measurement $y = h_1(\boldsymbol{x})+ h_2(\boldsymbol{x})+\epsilon$ at point $\bs x$, where $\epsilon$ is additive measurement noise independent of $h_1$ and $h_2$. Let 
	$\boldsymbol{y}_X = \boldsymbol{h}_{1,X}+\boldsymbol{h}_{2,X} +\boldsymbol{\epsilon}$ be the vector of such measurements at sampling points in a set $X$, where $\bs \epsilon$ is the vector of i.i.d. measurement noise. Then, 
	\[
	I(\boldsymbol{y}_X  ; h_1 + h_2) \leq I(\boldsymbol{h}_{1,X}+\boldsymbol{\epsilon};h_1) + I(\boldsymbol{h}_{2,X}+\boldsymbol{\epsilon};h_2).
	\]
\end{lemma}
\begin{proof}
This result can established by applying the data processing inequality~\cite[Theorem 2.8.1]{cover2012elements}.
\end{proof}

Let $\gamma_n^{m}$ be the maximal mutual information gain at fidelity $m$. It follows from Lemma~\ref{lemma: compmi} and the multi-fidelity GP model in~\eqref{def: mfgp} that $\gamma_n^{m} \leq \sum_{i=1}^{m} \gamma_n(l_i)$. Combining this inequality with Lemma~\ref{lemma: parmi}, we obtain the following result. 
\begin{corollary}[\bit{{Information gain for multi-fidelity GPs}}]\label{corollary: mfmi}
	The maximal mutual information gain at fidelity $m$ satisfies
	\begin{align*}
 \gamma_{n}^{m} \in O\Big(\sum_{i=1}^m l_i^{-2}(\log n)^3\Big).  
	\end{align*}
	
\end{corollary}
This corollary gives us an insight on the size of $\gamma_{n}^{m}$ at different fidelity level. 
 It follows that $\gamma_n^{(m)}$ grows faster at higher fidelity levels. 
 
We now derive a bound on the posterior variance for the multi-fidelity GP in terms of the maximum mutual information gain. 

\begin{lemma}[\bit{Uncertainty reduction for multi-fidelity GPs}]	\label{lemma: vub}
	Let $f \sim GP\left(\mu_0(\boldsymbol{x}), k_0(\boldsymbol{x},\boldsymbol{x}')\right)$ and $\sigma_0^2(\boldsymbol{x}) \leq \sigma^2 $, for each $\boldsymbol{x} \in D$. An additive sampling noise $\epsilon \sim N(0, s^2)$ is incurred every time $f$ is accessed. Under the greedy sampling policy the posterior variance after $n$ sampling rounds satisfies
	\begin{equation*}
	\max_{\boldsymbol{x}\in D} \sigma_n^2 (\boldsymbol{x}) \leq \frac{2 \sigma^2 }{\log \left( 1 + s^{-2}\sigma^2 \right)} \frac{\gamma_n}{n}.
	\end{equation*}
\end{lemma}
\begin{proof}
	For any $\boldsymbol{x} \in D$, $\sigma_{n}^2(\boldsymbol{x})$ is monotonically non-increasing in $n$. So we get
	\begin{equation}\label{ineq: variance}
		\max_{\boldsymbol{x}\in D} \sigma_{n}^2(\boldsymbol{x}) = \sigma_{n}^2(\boldsymbol{x}_{n+1}) \leq \sigma_{n-1}^2(\boldsymbol{x}_{n+1})\leq \sigma_{n-1}^2(\boldsymbol{x}_n),
	\end{equation}
	where the second inequality is due to the fact $\boldsymbol{x}_n = \argmax_{\boldsymbol{x}\in D}\sigma_{n-1}^2(\boldsymbol{x})$. Again since $\boldsymbol{x}_{n+1} = \argmax_{\boldsymbol{x}\in D}\sigma_{n}^2(\boldsymbol{x})$, inequality \eqref{ineq: variance} also indicates that $\sigma_{n-1}^2(\boldsymbol{x}_n)$ is monotonically non-increasing. Hence, from~\eqref{eq:mutual-info}, $\log \left( 1 + s^{-2}\sigma_{n-1}^2\left(\boldsymbol{x}_n\right) \right) \leq 2 \subscr{I}{greedy} \left(\boldsymbol{y}_X ; f \right)/n \leq 2\gamma_n/n$.
	Since ${s^2}/{\log \left(1+s^2\right)}$ is an increasing function on $[0,\infty)$,
	$${\sigma_{n-1}^2\left(\boldsymbol{x}_n\right)} \leq \frac{\sigma^2}{\log \left( 1 + s^{-2}\sigma^2 \right)} {\log \left( 1 + s^{-2}\sigma_{n-1}^2\left(\boldsymbol{x}_n\right) \right)}.$$
	Substituting~\eqref{ineq: variance} into it, we conclude that
	\begin{align*}
	\max_{\boldsymbol{x}\in D} \sigma_{n}^2(\boldsymbol{x})\leq  {\sigma_{n-1}^2\left(\boldsymbol{x}_n\right)} \leq \frac{2 \sigma^2 }{\log \left( 1 + s^{-2} \sigma^2  \right)} \frac{\gamma_n}{n}.
	\end{align*}
\end{proof}

Lemma~\ref{lemma: vub} indicates that the smaller and the more slowly growing $\gamma_n$ is, the faster $\max_{\boldsymbol{x}\in D} \sigma_{n}(\boldsymbol{x})$ converges.  This result explains our idea of using multi-fidelity model.

\subsection{Analysis of Expected Detection Time}
We now derive an upper bound on the number of samples needed to classify a location using EMTS algorithm and then use this result to compute the total sampling and travel time required for classification. 

\begin{lemma}[\bit{Sample complexity for uncertainty reduction}]\label{lemma: converge}
	In the autoregressive multi-fidelity model~\eqref{def: stker}, if each $h^{(m)}$ has a squared exponential kernel, then
	\[
	\min \setdef{n \in \natural}{\max_{\boldsymbol{x}\in D} \sigma_n (\boldsymbol{x})\leq \Delta} \in O \left(\frac{\sigma_0^2}{\Delta^2} \left(\ln \frac{\sigma_0}{\Delta}\right)^3 \right).
	\]
\end{lemma}
\begin{proof}
It follows from Lemma~\ref{lemma: vub} that 
	\begin{align*}
	\frac{n}{\gamma_n}\leq \frac{2 \sigma_0^2 }{\max_{\boldsymbol{x}\in D} \sigma_{n}^2(\boldsymbol{x})}.
	\end{align*}
	Since $v_m, s_m$ and $l_m$ for all fidelity levels are finite, it follows from Corollary~\ref{corollary: mfmi} that $\gamma_n \in O((\ln n)^3)$. Combining these results, the lemma follows by inspection. 
\end{proof}

\begin{lemma}[\bit{Sample complexity for EMTS}]\label{lemma: sampling rounds}
	For a given misclassification tolerance $\delta$, let $n(\boldsymbol{x}, \delta)$ be the number of samples required to classify $\boldsymbol{x} \in D$. Then, the expected number of samples satisfies	
	\[\expt [n(\boldsymbol{x}, \delta) \,|\, \Delta(\boldsymbol{x})] \in O \left(\varphi(\Delta(\boldsymbol{x}), \delta) \left(\ln \varphi(\Delta(\boldsymbol{x}), \delta)\right)^3 \right), \]
	where  $\Delta(\boldsymbol{x}) =\abs{f(\boldsymbol{x})-\texttt{th}}$ and $\varphi(\Delta(\boldsymbol{x}), \delta) = \frac{\sigma_0^2}{ \Delta^2(\boldsymbol{x})} {\ln \left(\frac{3\sigma_0}{\delta \Delta(\boldsymbol{x})} \right)}$.	
\end{lemma}

\begin{proof}	
	Since $\delta<1/2$, function $c(\delta/2^j) \left(3/4\right)^{j+1}$ is monotonically decreasing for $j\geq 2$.
	%
	%
	We define
	\[J = \ceil{ \log_{4/3} \left(\frac{3 \sigma_0}{ \Delta(\boldsymbol{x})} \sqrt{2\ln \left(\frac{3\sigma_0}{\delta \Delta(\boldsymbol{x})} \right)} \right) } + 1.\]	 
It can be shown that the choice of $J$ ensures, for $j \ge J$,
	\begin{align}
	U (\boldsymbol{x}, &{\delta}/{2^j}) - L (\boldsymbol{x},{\delta}/{2^j}) \nonumber\\
	\leq & \; 2 c(\delta/2^j) \left(3/4\right)^{j+1}\sigma_0 \leq 2 c(\delta/2^J) \left(3/4\right)^{J+1}\sigma_0 \nonumber\\
	\leq & \; \frac{\Delta(\boldsymbol{x})}{2} \sqrt{\frac{{ {\alpha} \ln \left( \frac{3\sigma}{\delta\Delta(\boldsymbol{x})} \sqrt{2\ln \frac{3\sigma}{\delta \Delta(\boldsymbol{x}) }}\right) }}{\ln \frac{3\sigma}{\delta \Delta(\boldsymbol{x})}} } < {\Delta(\boldsymbol{x})} \label{region}
	\end{align}
	where $\alpha = \log_{4/3} 2$ and the second inequality is due to the fact $\ln(x\ln(x))/\ln(x) \leq (1+e)/e$. For a point $\boldsymbol{x}$ at which $c^*(\boldsymbol{x}) = 1$ and $\Delta(\boldsymbol{x}) > 0$, based on~\eqref{region}, the number of sampling rounds to classify $\boldsymbol{x}$ satisfies
	\begin{align*}
	n(\boldsymbol{x},\delta) &\leq n_J + \sum_{j=J+1}^{\infty} \indicator{L (\boldsymbol{x},{\delta}/{2^j} )< \texttt{th} \leq U(\boldsymbol{x},{\delta}/{2^j} ) } \\
	&\leq n_J + \sum_{j=J+1}^{\infty} \indicator{L (\boldsymbol{x},{\delta}/{2^j} )< \texttt{th} } \\
	&\leq n_J + \sum_{j=J+1}^{\infty} \indicator{U (\boldsymbol{x},{\delta}/{2^j} ) < f(\boldsymbol{x}) },
	\end{align*}
	where $n_J$ is the number of samples collected in the first $J$ epochs. Then the expected sampling rounds can be bounded as
	\begin{align*}
	\bar n (\boldsymbol{x},\delta) &\leq n_J + \sum_{j=J+1}^{\infty} \prob\left( L (\boldsymbol{x},{\delta}/{2^j}) \geq \texttt{th} \right) \\
	&\leq n_J + \sum_{j=J+1}^{\infty} \prob\left( L (\boldsymbol{x},{\delta}/{2^j}) \geq \texttt{th} \right)\\
	& \leq n_J + \sum_{j = 1}^{\infty} \frac{n_j}{2^j}.
	\end{align*}
	%
	From Lemma~\ref{lemma: converge}, we has $n_i \in \tilde O((16/9)^j )$. Therefore $\sum_{j = 1}^{\infty} {n_j}/{2^j}$ is finite. So we conclude
	\[\bar n (\boldsymbol{x},\delta) \in O \left(\varphi(\Delta(\boldsymbol{x}), \delta) \left(\ln \varphi(\Delta(\boldsymbol{x}), \delta)\right)^3 \right) . \]
	
\end{proof}

\begin{remark}(\bit{Comparison with sample complexity of multiarmed bandits:})
	Notice that $\expt [n(\boldsymbol{x}, \delta) \,|\, \Delta(\boldsymbol{x})] \in \tilde{O} \left(\frac{1}{ \Delta^2(\boldsymbol{x})}\right)$ describes the complexity to of classification of $\boldsymbol{x}$, i.e., for a point with $f(\boldsymbol{x})$ close to $\texttt{th}$ more time is needed. This term is similar to the sampling complexity~\cite{mannor2004sample} in a pure-exploration multi-armed bandit problem. This result is based on the assumption that GPs all have squared exponential kernel. For kernels characterizing less correlations, e.g. Mart\'en kernels, more sampling rounds are expected. \oprocend
\end{remark}

We now derive an upper-bound on expected detection time for EMTS.

\begin{theorem}[\bit{Expected classification time for EMTS}]~\label{theorem: est}
	For a location $\boldsymbol{x} \in D$ and misclassification tolerance $\delta$, the expected classification time for $\boldsymbol{x}$ satisfies
	\[\bar t(\boldsymbol{x},\delta) \, \in O\left(\varphi(\Delta(\boldsymbol{x}), \delta) \left(\ln \varphi(\Delta(\boldsymbol{x}), \delta)\right)^3 \right),\] 
	where $\varphi(\Delta(\boldsymbol{x}), \delta) = \frac{\sigma_0^2}{ \Delta^2(\boldsymbol{x})} {\ln \left(\frac{3\sigma_0}{\delta \Delta(\boldsymbol{x})} \right)}$.
\end{theorem}
\smallskip
\begin{proof}
Since we assume unit sampling time, the total sampling time is in the same order as $n(\boldsymbol{x}, \delta)$. Then we consider the traveling time spent in order to collected those samples. Since EMTS requires the vehicle to search from low fidelity level to high fidelity level, the total number of altitude switches is no greater than $M-1$. As presented in~\cite{karloff1989long}, for $n$ points in $[0,1]^2$, the length of the shortest TSP Tour $< 0.984\sqrt{2 n} +11$. Therefore, the expected traveling time belongs to  $O\big(d\sqrt{\bar n (\boldsymbol{x},\delta)}\big)$, where $d$ is the diameter of $D$.
Thus, the expected traveling time belongs to $o(\bar n (\boldsymbol{x},\delta))$. Considering both sampling and traveling time, we conclude
\[\bar t(\boldsymbol{x},\delta) \in  O \left(\varphi(\Delta(\boldsymbol{x}), \delta) \left(\ln \varphi(\Delta(\boldsymbol{x}), \delta)\right)^3 \right). \]
\end{proof}
Theorem~\ref{theorem: est} illustrates the efficiency of the EMTS algorithm, we conjecture it to be near optimal. It has a natural implication that the expected classification time at a location increases with the classification complexity and the desired classification accuracy. 


\section{Conclusions and Future Directions}\label{sec:conclusions}
In this paper, we extended the classical informative path planning approach for single-fidelity GPs to multi-fidelity GPs. This novel extension allowed for jointly planning for sampling locations and associated fidelity-levels, and thus, addresses the fidelity-coverage trade-off. We proposed and analyzed the EMTS algorithm for multi-target search that yields sampling points that the robot should visit and the fidelity level with which the robot should collect the information at these points. We illustrated our algorithm in an underwater victim search scenario using the Unmanned Underwater Vehicle Simulator. We rigorously analyzed the algorithm in terms of its accuracy in classifying the locations in the environment as empty or occupied by a target, as well as the expected time the robot takes to classify these points. 

Future research include the extension to cooperative multi-robot search scenarios and implementation of the proposed algorithm in our underwater multi-target search testbed.

\footnotesize 

\bibliographystyle{IEEEtran}
\bibliography{IEEEabrv,bandits,surveillance,mybib}

\begin{thebibliography}{10}
\providecommand{\url}[1]{#1}
\csname url@samestyle\endcsname
\providecommand{\newblock}{\relax}
\providecommand{\bibinfo}[2]{#2}
\providecommand{\BIBentrySTDinterwordspacing}{\spaceskip=0pt\relax}
\providecommand{\BIBentryALTinterwordstretchfactor}{4}
\providecommand{\BIBentryALTinterwordspacing}{\spaceskip=\fontdimen2\font plus
\BIBentryALTinterwordstretchfactor\fontdimen3\font minus
  \fontdimen4\font\relax}
\providecommand{\BIBforeignlanguage}[2]{{%
\expandafter\ifx\csname l@#1\endcsname\relax
\typeout{** WARNING: IEEEtran.bst: No hyphenation pattern has been}%
\typeout{** loaded for the language `#1'. Using the pattern for}%
\typeout{** the default language instead.}%
\else
\language=\csname l@#1\endcsname
\fi
#2}}
\providecommand{\BIBdecl}{\relax}
\BIBdecl

\bibitem{NEL-DAP-etal:07}
N.~E. Leonard, D.~A. Paley, F.~Lekien, R.~Sepulchre, D.~M. Fratantoni, and
  R.~E. Davis, ``Collective motion, sensor networks, and ocean sampling,''
  \emph{Proceedings of the IEEE}, vol.~95, no.~1, pp. 48--74, 2007.

\bibitem{SLS-MS-DR:12}
S.~L. Smith, M.~Schwager, and D.~Rus, ``Persistent robotic tasks: Monitoring
  and sweeping in changing environments,'' \emph{IEEE Transactions on
  Robotics}, vol.~28, no.~2, pp. 410--426, 2012.

\bibitem{CGC-XL-XD:13}
C.~G. Cassandras, X.~Lin, and X.~Ding, ``An optimal control approach to the
  multi-agent persistent monitoring problem,'' \emph{IEEE Transactions on
  Automatic Control}, vol.~58, no.~4, pp. 947--961, 2013.

\bibitem{RNS-MS-etal:11}
R.~N. Smith, M.~Schwager, S.~L. Smith, B.~H. Jones, D.~Rus, and G.~S. Sukhatme,
  ``Persistent ocean monitoring with underwater gliders: Adapting sampling
  resolution,'' \emph{Journal of Field Robotics}, vol.~28, no.~5, pp. 714--741,
  2011.

\bibitem{AK-CEG:12}
A.~Krause and C.~E. Guestrin, ``Near-optimal nonmyopic value of information in
  graphical models,'' in \emph{Proceedings of the Twenty-First Conference
  Conference on Uncertainty in Artificial Intelligence}, Edinburgh, Scotland,
  Jul. 2005, pp. 324--331.

\bibitem{williams2006gaussian}
C.~K. Williams and C.~E. Rasmussen, \emph{Gaussian processes for Machine
  Learning}.\hskip 1em plus 0.5em minus 0.4em\relax MIT press Cambridge, MA,
  2006, vol.~2, no.~3.

\bibitem{SV-FR-EN-HD:09}
S.~Vasudevan, F.~Ramos, E.~Nettleton, and H.~Durrant-Whyte, ``{G}aussian
  process modeling of large-scale terrain,'' \emph{Journal of Field Robotics},
  vol.~26, no.~10, pp. 812--840, 2009.

\bibitem{AS-AK-CG-WJK:09}
A.~Singh, A.~Krause, C.~Guestrin, and W.~J. Kaiser, ``Efficient informative
  sensing using multiple robots,'' \emph{Journal of Artificial Intelligence
  Research}, vol.~34, no.~2, p. 707, 2009.

\bibitem{krause2008near}
A.~Krause, A.~Singh, and C.~Guestrin, ``Near-optimal sensor placements in
  gaussian processes: Theory, efficient algorithms and empirical studies,''
  \emph{Journal of Machine Learning Research}, vol.~9, no. Feb, pp. 235--284,
  2008.

\bibitem{JLN-GJP:09}
J.~L. Ny and G.~J. Pappas, ``On trajectory optimization for active sensing in
  {G}aussian process models,'' in \emph{IEEE Conf\ on Decision and Control and
  Chinese Control Conference}, Shanghai, China, Dec. 2009, pp. 6286--6292.

\bibitem{XL-MS:13}
X.~Lan and M.~Schwager, ``Planning periodic persistent monitoring trajectories
  for sensing robots in {G}aussian random fields,'' in \emph{IEEE Int\ Conf\ on
  Robotics and Automation}, Karlsruhe, Germany, May 2013, pp. 2415--2420.

\bibitem{DES-MS-DS:12}
D.~E. Soltero, M.~Schwager, and D.~Rus, ``Generating informative paths for
  persistent sensing in unknown environments,'' in \emph{IEEE/RSJ Int\ Conf\ on
  Intelligent Robots and Systems}, Vilamoura, Algarve, Portugal, Oct. 2012, pp.
  2172--2179.

\bibitem{JY-MS-DR:14}
J.~Yu, M.~Schwager, and D.~Rus, ``Correlated orienteering problem and its
  application to informative path planning for persistent monitoring tasks,''
  pp. 342--349, 2014.

\bibitem{VS-FP-FB:11za}
V.~Srivastava, F.~Pasqualetti, and F.~Bullo, ``Stochastic surveillance
  strategies for spatial quickest detection,'' \emph{The International Journal
  of Robotics Research}, vol.~32, no.~12, pp. 1438--1458, 2013.

\bibitem{VS-PR-NEL:14}
V.~Srivastava, P.~Reverdy, and N.~E. Leonard, ``Surveillance in an abruptly
  changing world via multiarmed bandits,'' in \emph{IEEE Conference on Decision
  and Control}, 2014, pp. 692--697.

\bibitem{hollinger2014sampling}
G.~A. Hollinger and G.~S. Sukhatme, ``Sampling-based robotic information
  gathering algorithms,'' \emph{The International Journal of Robotics
  Research}, vol.~33, no.~9, pp. 1271--1287, 2014.

\bibitem{GAH-BE-etal:13}
G.~A. Hollinger, B.~Englot, F.~S. Hover, U.~Mitra, and G.~S. Sukhatme, ``Active
  planning for underwater inspection and the benefit of adaptivity,'' \emph{The
  International Journal of Robotics Research}, vol.~32, no.~1, pp. 3--18, 2013.

\bibitem{hitz2017adaptive}
G.~Hitz, E.~Galceran, M.-{\`E}. Garneau, F.~Pomerleau, and R.~Siegwart,
  ``Adaptive continuous-space informative path planning for online
  environmental monitoring,'' \emph{Journal of Field Robotics}, vol.~34, no.~8,
  pp. 1427--1449, 2017.

\bibitem{hitz2014fully}
G.~Hitz, A.~Gotovos, M.-{\'E}. Garneau, C.~Pradalier, A.~Krause, R.~Y. Siegwart
  \emph{et~al.}, ``Fully autonomous focused exploration for robotic
  environmental monitoring,'' in \emph{2014 IEEE International Conference on
  Robotics and Automation (ICRA)}.\hskip 1em plus 0.5em minus 0.4em\relax IEEE,
  2014, pp. 2658--2664.

\bibitem{atanasov2014information}
N.~Atanasov, J.~Le~Ny, K.~Daniilidis, and G.~J. Pappas, ``Information
  acquisition with sensing robots: Algorithms and error bounds,'' in \emph{2014
  IEEE International Conference on Robotics and Automation (ICRA)}.\hskip 1em
  plus 0.5em minus 0.4em\relax IEEE, 2014, pp. 6447--6454.

\bibitem{meera2019obstacle}
A.~A. Meera, M.~Popovi{\'c}, A.~Millane, and R.~Siegwart, ``Obstacle-aware
  adaptive informative path planning for uav-based target search,'' in
  \emph{2019 International Conference on Robotics and Automation (ICRA)}.\hskip
  1em plus 0.5em minus 0.4em\relax IEEE, 2019, pp. 718--724.

\bibitem{sung2019environmental}
Y.~Sung, D.~Dixit, and P.~Tokekar, ``Environmental hotspot identification in
  limited time with a uav equipped with a downward-facing camera,'' \emph{arXiv
  preprint arXiv:1909.08483}, 2019.

\bibitem{srinivas2012information}
N.~Srinivas, A.~Krause, S.~M. Kakade, and M.~W. Seeger, ``Information-theoretic
  regret bounds for gaussian process optimization in the bandit setting,''
  \emph{IEEE Transactions on Information Theory}, vol.~58, no.~5, pp.
  3250--3265, 2012.

\bibitem{PR-VS-NEL:13d}
P.~Reverdy, V.~Srivastava, and N.~E. Leonard, ``Modeling human decision making
  in generalized {G}aussian multiarmed bandits,'' \emph{Proceedings of the
  IEEE}, vol. 102, no.~4, pp. 544--571, 2014.

\bibitem{chen2014combinatorial}
S.~Chen, T.~Lin, I.~King, M.~R. Lyu, and W.~Chen, ``Combinatorial pure
  exploration of multi-armed bandits,'' in \emph{Advances in Neural Information
  Processing Systems}, 2014, pp. 379--387.

\bibitem{PR-VS-NEL:14h}
P.~Reverdy, V.~Srivastava, and N.~E. Leonard, ``Satisficing in multi-armed
  bandit problems,'' \emph{IEEE Transactions on Automatic Control}, vol.~62,
  no.~8, pp. 3788 -- 3803, 2017.

\bibitem{kennedy2000predicting}
M.~C. Kennedy and A.~O'Hagan, ``Predicting the output from a complex computer
  code when fast approximations are available,'' \emph{Biometrika}, vol.~87,
  no.~1, pp. 1--13, 2000.

\bibitem{kandasamy2016gaussian}
K.~Kandasamy, G.~Dasarathy, J.~B. Oliva, J.~Schneider, and B.~P{\'o}czos,
  ``Gaussian process bandit optimisation with multi-fidelity evaluations,'' in
  \emph{Advances in Neural Information Processing Systems}, 2016, pp.
  992--1000.

\bibitem{redmon2018yolov3}
J.~Redmon and A.~Farhadi, ``Yolov3: An incremental improvement,'' \emph{arXiv
  preprint arXiv:1804.02767}, 2018.

\bibitem{perdikaris2017gaussian}
\BIBentryALTinterwordspacing
P.~Perdikaris, ``Gaussian processes a hands-on tutorial,'' 2017. [Online].
  Available: \url{https://github.com/paraklas/GPTutorial}
\BIBentrySTDinterwordspacing

\bibitem{kemna2017multi}
S.~Kemna, J.~G. Rogers, C.~Nieto-Granda, S.~Young, and G.~S. Sukhatme,
  ``Multi-robot coordination through dynamic {V}oronoi partitioning for
  informative adaptive sampling in communication-constrained environments,'' in
  \emph{2017 IEEE International Conference on Robotics and Automation
  (ICRA)}.\hskip 1em plus 0.5em minus 0.4em\relax IEEE, 2017, pp. 2124--2130.

\bibitem{applegate2006concorde}
D.~Applegate, R.~Bixby, V.~Chvatal, and W.~Cook, ``Concorde {TSP} solver,''
  2006.

\bibitem{audibert2010best}
J.-Y. Audibert and S.~Bubeck, ``Best arm identification in multi-armed
  bandits,'' in \emph{COLT}, 2010, pp. 13--p.

\bibitem{rolf2018successive}
E.~Rolf, D.~Fridovich-Keil, M.~Simchowitz, B.~Recht, and C.~Tomlin, ``A
  successive-elimination approach to adaptive robotic sensing,'' \emph{ArXiv
  e-prints}, 2018.

\bibitem{manhaes2016uuv}
M.~M.~M. Manh{\~a}es, S.~A. Scherer, M.~Voss, L.~R. Douat, and T.~Rauschenbach,
  ``{UUV} simulator: A {G}azebo-based package for underwater intervention and
  multi-robot simulation,'' in \emph{OCEANS 2016 MTS/IEEE Monterey}.\hskip 1em
  plus 0.5em minus 0.4em\relax IEEE, 2016, pp. 1--8.

\bibitem{MA-IAS:64}
M.~Abramowitz and I.~A. Stegun, Eds., \emph{Handbook of Mathematical Functions:
  with Formulas, Graphs, and Mathematical Tables}.\hskip 1em plus 0.5em minus
  0.4em\relax Dover Publications, 1964.

\bibitem{nemhauser1978analysis}
G.~L. Nemhauser, L.~A. Wolsey, and M.~L. Fisher, ``An analysis of
  approximations for maximizing submodular set functions,'' \emph{Mathematical
  programming}, vol.~14, no.~1, pp. 265--294, 1978.

\bibitem{cover2012elements}
T.~M. Cover and J.~A. Thomas, \emph{Elements of Information Theory}.\hskip 1em
  plus 0.5em minus 0.4em\relax John Wiley \& Sons, 2012.

\bibitem{mannor2004sample}
S.~Mannor and J.~N. Tsitsiklis, ``The sample complexity of exploration in the
  multi-armed bandit problem,'' \emph{Journal of Machine Learning Research},
  vol.~5, no. Jun, pp. 623--648, 2004.

\bibitem{karloff1989long}
H.~J. Karloff, ``How long can a euclidean traveling salesman tour be?''
  \emph{SIAM Journal on Discrete Mathematics}, vol.~2, no.~1, pp. 91--99, 1989.

\end{thebibliography}

\end{document}